\documentclass{article}
\usepackage{ijcai20}

\usepackage{times}
\usepackage{soul}
\usepackage{url}
\usepackage[hidelinks]{hyperref}
\usepackage[utf8]{inputenc}
\usepackage[small]{caption}
\usepackage{graphicx}
\usepackage{amsmath}
\usepackage{amsthm}
\usepackage{amsfonts}
\usepackage{booktabs}
\usepackage{algorithm}
\usepackage{algorithmic}
\usepackage{fancyvrb}
\usepackage{color}
\usepackage{wrapfig}

\newtheorem{theorem}{Theorem}
\newtheorem{definition}{Definition}
\newtheorem{lemma}{Lemma}

\title{Decentralized MCTS via Learned Teammate Models}

\author{
Aleksander Czechowski
{\normalfont and}
Frans A. Oliehoek
\affiliations
Delft University of Technology
\emails
\{a.t.czechowski, f.a.oliehoek\}@tudelft.nl
}

\begin{document}

\maketitle

\begin{abstract}
Decentralized online planning can be an attractive paradigm for cooperative 
multi-agent systems, due to improved scalability and robustness.
A key difficulty of such approach lies in making accurate
predictions about the decisions of other agents.  
In this paper, we present a trainable online decentralized planning algorithm
based on decentralized Monte Carlo Tree Search, combined with
models of teammates learned from previous episodic runs. 
By only allowing one agent to adapt its models at a time, 
under the assumption of ideal policy approximation,
successive iterations of our method are guaranteed to improve joint
policies, and eventually lead to convergence to a Nash equilibrium.  
We test the efficiency of the algorithm by performing experiments 
in several scenarios of the spatial task
allocation environment introduced in~\cite{claes2015effective}.  We show that
deep learning and convolutional neural networks can be employed
to produce accurate policy approximators which exploit the spatial features of the
problem, and that the proposed algorithm improves over the baseline
planning performance for particularly challenging domain configurations.
\end{abstract}

\section{Introduction}
The ability to compute or learn plans to realize complex tasks is a central question in artificial intelligence.
In the case of multi-agent systems, the coordination problem is of utmost importance:
how can teams of artificial agents be engineered to work together, to achieve a common goal?
A decentralized approach to this problem has been adopted in many techniques~\cite{durfee2013multiagent}.
The motivation comes from human collaboration: in most contexts we plan individually, 
and in parallel with other humans.
Moreover, decentralized planning method can lead to a number of benefits, 
such as robustness, reduced computational load and 
absence of communication overhead~\cite{claes2017decentralised}. 

Decentralized planning methods were applied in context of
multiplayer computer games~\cite{jaderberg2019human}, robot soccer~\cite{acsik2012solving}, 
intersection control~\cite{vu2018decentralised} 
and autonomous warehouse control~\cite{claes2017decentralised}, to name a few.
The essential difficulty of this paradigm lies in solving the coordination problem.
To naively deploy single-agent algorithms for individual agents inevitably leads to the tragedy
of the commons; i.e. a situation where 
an action that seems optimal from an individual perspective, is suboptimal collectively.
For instance, consider a relatively simplistic instance of a spatial task allocation problem
in which a team of $n$ robotic vacuum cleaners needs to clean a factory floor, as in Figure~\ref{facfloor}.
\begin{figure}[t]
  \centering
  \includegraphics[width=60mm]{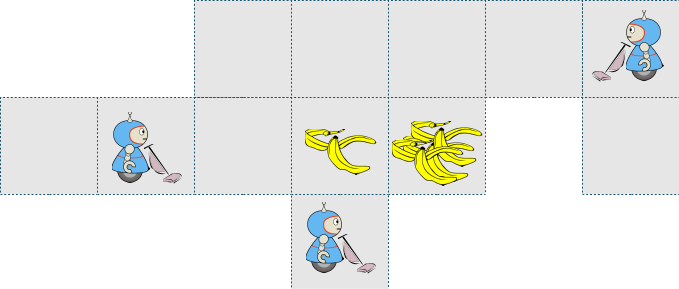}
  \caption{Robots cleaning a factory floor.}\label{facfloor}
\end{figure}
Assuming that a robot solves its own traveling salesman problem~\cite{lin1965computer} 
would result in optimal path planning if it was alone in the factory;
but collectively it could lead to unnecessary duplication of resources
with multiple robots heading to the same littered area.
On the other hand, joint optimization of all actions results in an intractable problem,
that is not scalable to large networks of agents.
Among some of the heuristic methods to deal with such problems proposed by researchers,
communication~\cite{wu2009multi}, higher level coordination orchestration~\cite{borrajo2019efficient},
and co-agent modelling~\cite{albrecht2018autonomous}
were previously explored in literature.

The cooperative decentralized planning problem can be posed in different settings;
within this paper we focus on \emph{simulation-based planning}, 
where each agent in the team has access to a simulator of the environment,
which they can use to sample states and rewards, and 
evaluate the value of available actions, before committing to a particular one.
The inherent difficulty of decentralized simulation-based planning
is that in order for an individual agent to sample from the simulator
and estimate the potential future rewards,
it needs to provide \emph{joint actions} of themselves and their teammates.
However, in a live planning scenario, where each of the agents chooses actions
according to their own simulation-based algorithm, 
it is not possible to know \emph{a priori} what actions teammates actually execute.

There are two basic approaches to deal with this. 
If all agents are deployed with the same algorithm, they can 
evaluate all joint actions and choose their respective individual action;
however this approach is costly, and the computational difficulty 
grows exponentially with the number of agents.
A different approach is to make assumptions on other agents,
and supply the simulator with an educated guess on their actions, given the common observed state.
Such solution was used in~\cite{claes2017decentralised}, where heuristic policies were designed
for a domain modelling the task allocation problem in a factory floor.

In this paper, we build upon the second paradigm.
We introduce a decentralized planning method of \emph{Alternate maximization with Behavioural Cloning} (ABC).
Our algorithm combines the ideas of alternate maximization,
behavioral cloning and Monte Carlo Tree Search (MCTS) in a previously unexplored manner.
By the ABC method, the agents \emph{learn} the behavior of their teammates, 
and adapt to it in an iterative manner.
The high-level overview of our planning-execution loop for a team of agents in a given environment
can be represented in the following alternating steps:
\begin{enumerate}
  \item We perform a number of episode simulations 
    with agents acting according to their individual MCTS; 
    each agent has an own simulator of the environment and models of its teammates;
\item 
    The data from the simulations (in the form of state-action pairs for all agents) 
    is used to train new agent behaviour models;
    these are in turn inserted into the MCTS simulator of one of the agents.
\end{enumerate}
We refer to each successive iteration of above two steps as a \emph{generation}.
In each generation we chose a different agent for the simulator update.

We prove that if original policies have been perfectly replicated by learning, 
we are guaranteed to increase the mean total reward
at each step, and eventually converge to a Nash equilibrium.
We also demonstrate the empirical value of our method by 
experimental evaluation in the previously mentioned factory floor domain.

\section{Related Work}
In this paper, we take a so-called \emph{subjective
perspective}~\cite{Oliehoek16Book} of the multi-agent scenario, in which the system is
modeled from a protagonist agent's view. The simplest approach simply ignores the other
agents completely (some-times called `self-absorbed' \cite{claes2015effective}). On the
complex end of the spectrum, there are `intentional model' approaches that recursively model the other agents,
such as the recursive modeling method \cite{Gmytrasiewicz95ICMAS}, and interactive POMDPs
\cite{gmytrasiewicz2005framework,doshi2009monte,eck2019scalable}. In between lie other,
often simpler, forms of modeling other agents \cite{HernandezLeal17arxiv,albrecht2018autonomous,Hernandez19AIIDE}. Such models can be tables, heuristics, finite-state machines, neural networks, other machine learning
models, etc. Given that these do not explicitly model other agents, they have been called
`sub-intentional', however, as demonstrated by \cite{Rabinowitz18ICML} they can
demonstrate complex characteristics associated with the `theory of mind' \cite{Premack78ToM}.

In our approach we build on the idea that (sub-intentional) neural network models can
indeed provide accurate models of behaviors of teammates. 
Contrary to~\cite{claes2017decentralised}, who couple MCTS with heuristics to predict the
teammates, this makes our method domain independent as
well as providing certain guarantees (assuming ideal policy replication by function approximators). Recursive
methods, such as the interactive particle filter \cite{doshi2009monte} also give certain
guarantees, but are typically based on a finite amount of recursion, which means that at
the lowest level they make similar heuristic assumptions. Another drawback is their very
high computational cost, making them very difficult to apply on-line.

In~\cite{kurzer2018decentralized} decentralized MCTS is combined with macro-actions
for automated vehicle trajectory planning.
The authors however assume heuristic (so not always accurate) 
models of other agents and do not learn their actual policies.
\cite{best2016decentralised} uses parallel MCTS
for an active perception task and combines it with communication to solve the coordination problem;
\cite{li2019integrating} explores similar ideas and combines communication with heuristic teammate models
of~\cite{claes2017decentralised}.
Contrary to both of these papers, we assume no communication during the execution phase.
Similarly, \cite{golpayegani2015collaborative} uses MCTS in a parallel, decentralized way,
but includes a so-called Collaborative Stage at each decision making point, where agents
can jointly agree on the final decision. 
Other research has tried to make joint-action MCTS more scalable by exploiting factored state spaces~\cite{amato2015scalable}.
In our setting,
we do not make any assumptions about the dynamics of the environment.

Finally, we would like to point out similarity of our approach with
AlphaGo \& AlphaZero -- the computer programs
designed to master the game of Go in a highly acclaimed research by Deepmind~\cite{silver2016mastering,Silver17nature}.
There, neural network models were used together with self-play to guide MCTS,
by providing guesses of opponents gameplay and estimates on state-action value functions.
However, both AlphaGo \& AlphaZero expand opponents' actions in the search tree.
By our approach, we are able to incorporate the actions of other agents in the environment simulator
so they do not contribute to the branching factor of the decision tree,
which, in turn, allows us to scale the method to several agents.

\section{Background}

A Markov Decision Process (MDP)
is defined as a 5-tuple $M:=(S,A,T,R,\gamma)$, where $S$ is a finite set of states,
$A$ is a finite set of actions,
$T: S\times A \times S \to [0,1]$ are probabilities of transitioning between states for particular choices of actions, 
$R: S\times A \to \mathbb{R}$ is a reward function 
and $\gamma \in [0,1]$ is the discount factor.
The \emph{policy} is a mapping $\pi: S \to A$, 
which represents an action selection rule for the agent.
The policy is paired with the environment to form a Markov chain over the state space
defined by the sequence of probability distribution functions
which models the process of decision making.
The value of a state is given by
\begin{equation}
  V^{M,\pi}(s_0) = \sum_{t} \gamma^t \mathbb{E}(R_t | \pi, s_0)
\end{equation}
One is typically interested in finding the \emph{optimal policy},
$\pi^{*,M} := \operatorname{argmax}_{\pi} V^{M,\pi}$.
The value of action $a$ in a given state $s$ is given by the $Q$ function
\begin{equation}
  Q^{M,\pi}(s,a):=  R(s,a) + \gamma\sum_{s'}T(s'|s,a) V^{M,\pi^{*}} (s').\\
\end{equation}
By the Bellman optimality principle, the actions with highest Q values form the optimal policy
$\pi^{*}(s,a) = \operatorname{argmax}_{a \in A} Q(s,a)$.

In a Multi-agent Markov Decision Process (MMDP) with $n$ agents the action space is factored in $n$ components: $A=A_1 \times \dots \times A_n$.
Each component $A_{i}$ describes the individual actions available to the agent $i$ and
the policies can be represented as products of individual agent policies $\pi=(\pi_1, \dots, \pi_n)$.
We emphasize that an MMDP is fully cooperative and all agents receive the same reward signal $R$
(contrary to e.g. stochastic games, where the rewards are individual).
For our considerations, it will be useful to introduce the $i$-th self-absorbed projection
of an MMDP $M=(S,A,T,R,\gamma)$, after having fixed all individual policies besides $\pi_i$, as a single agent MDP:
\begin{equation}
  \Pi_{i}(M, \pi_{-i}) := (S,A_i,T_i,R,\gamma)\\
\end{equation}
where $\pi_{-i}$ denotes (fixed) policies of all agents except for agent $i$ and the 
transitions and the rewards are induced by compositions of $T,R$ with $\pi_{-i}$.

The problem of finding solutions to an MMDP can be viewed as a collaborative normal form game~\cite{claus1998dynamics,peshkin2000learning}, 
where the agents are players, the individual policies are strategies,
and the payoffs for a joint policy $\pi$ and an initial state $s_0$ 
are given by $V^{M,\pi}(s_0)$, and uniform to all players.
A joint policy is a \emph{Nash equilibrium} if and only if no higher payoff can be achieved by changing 
only one of the individual policies forming it.
The optimal policy\footnote{For ease of exposition of the theoretical background, we assume it to be unique; 
this is not a restriction on the class of MDPs as one can always
perturb the reward function in an otherwise negligible fashion to make policy values disjoint.} $\pi^{*,M}$
is a Nash equilibrium,
however there may be multiple other (suboptimal) Nash equilibria in the system.

The question of finding an optimal joint policy can be considered in the simulation-based planning context.
There, it is no longer assumed that we have access to the full probabilistic model of the domain.
Instead, one is supplied with a \emph{system simulator}, i.e. a method for sampling new
states $s'$ and rewards $r$ based on states $s$ and joint actions $a$, according to the underlying
(but otherwise possibly unknown) probability distribution $T$ and reward function $R$.
In this paper, we consider the setting of online, decentralized, simulation-based planning,
where the agents need to compute individual best responses $\pi_i^{*}(s)$
over states $s \in S$ they encounter in the episode. 

We focus 
on one particularly effective and popular planning method: the MCTS algorithm
combined with the Upper Confidence Trees (UCT) tree exploration policy.
This search method uses Monte Carlo simulations
to construct a tree of possible future evolutions of the system.
The tree consists of nodes representing actions taken by the agent, and 
the resulting, sampled states encountered in the environment.
Each node stores statistics that approximate either the state values or the $Q$ values of actions.
The single iteration of algorithm execution is split into four parts.
First, the tree is traversed according to the tree policy (selection).
Then, new nodes are created by sampling an action and the resulting state (expansion).
Next, a heuristic policy is used to complete the episode simulation (rollout).
Finally, the results are stored in the visited tree nodes (backpropagation).

The selection step is performed by choosing
a node $k$ with an action $a$, which maximizes the formula
 $ \tilde{Q}(s,a,t)+c\sqrt{\frac{\log N_k}{n_k}}$,
where $\tilde{Q}$ is a sample-based estimator of the $Q$ value, $N_k$ is the 
amount of visits at the parent node of node $k$, and $n_k$ is the amount of visits to the node.
All of these three values are updated at each backpropagation step.
The constant $c>0$ is the exploration constant;
in theory, for rewards in $[0,1]$ range, it should be equal to $\sqrt{2}$.
In practice the constant is chosen empirically~\cite{kocsis2006bandit}.

The algorithm is initialized and performed at each time step of simulation,
for either a predefined, or time-limited amount of iterations.
Then, the best action is selected greedily, based on 
the approximate $Q$-values of child nodes of the root node.

\begin{definition}
We denote the policy generated by 
action selection according to the MCTS algorithm with UCT in an MDP $M$
by $MCTS(M)(=MCTS(M,C,l,\rho))$, with $C>0$ being the exploration constant, $l \in \mathbb{N}$ 
the number of UCT iterations and $\rho$ -- a rollout policy.
\end{definition}

For a sufficiently large number of iterations $l=l(C,M)$
the MCTS algorithm approximates the real $Q$-values of each action node with arbitrary accuracy;
and therefore it constitutes the pure, optimal policy: $MCTS(M,C,l,\rho)=\pi^{*,M}$, c.f.~\cite{chang2005adaptive}.

\section{Alternating Maximization with Behavioral Cloning}

In this section we describe the ABC algorithm and prove its 
convergence guarantees.

\subsection{The Hill Climb}


A common method for joint policy improvement in multi-agent decision making
is the so-called \emph{hill climb}, where agents alternate between 
improving their policies (c.f.~\cite{nair2003taming}).
At each iteration
of the method (i.e. generation), 
one of the agents is designated to compute its best response,
while the other agents keep their policies fixed. 
The hill climb method comes with performance guarantees, 
in particular the joint rewards are guaranteed to (weakly) increase in subsequent generations.

Consider an MMDP on $n$ agents $M=(S,A,T,R)$,
and let $(\pi_1, \dots, \pi_n)$ denote the individual components of a joint policy $\pi$.

\begin{definition}
For each $i \in \{1,\dots,n\}$ we define the $i$-th best response operator
$BR_i$ from the joint policy space to itself by:
\begin{equation}
  BR_{i}\left(\pi \right) := \left(\pi_1,\dots,\pi_{i-1}, \pi^{*, \Pi_i(M, \pi_{-i})}_{i}, \pi_{i+1}, \dots, \pi_n \right).
\end{equation}
\end{definition}

\begin{lemma}\label{lem1}
  The following inequality holds:
  \begin{equation}
  V^{M, BR_{i}(\pi)}(s) \geq V^{M, \pi}(s),\ \forall s,i.
  \end{equation}
  Moreover, $V^{M, BR_{i}(\pi)}(s) =  V^{M, \pi}(s) \ \forall s$ implies that
  $\pi$ is a fixed point of $BR_{i}$.
\end{lemma}

\begin{proof}
For all $s \in S$: 
\begin{equation}
  \begin{aligned}
  V^{M, BR_{i}(\pi)}(s) &= V^{\Pi_i(M,\pi_{-i}), (BR_{i}(\pi))_i}(s)\\
  &=V^{\Pi_i(M,\pi_{-i}), \pi_i^{*, \Pi_i(M, \pi_{-i})}}(s)\\
  &\geq V^{\Pi_i(M,\pi_{-i}), \pi_i}(s)\\
  &=V^{M, \pi}(s).
\end{aligned}
\end{equation}
If the above are equal for all $s \in S$, then $\pi_i = \pi_i^{*, \Pi_i(M,\pi_{-i})}$.
\end{proof}

Applications of Lemma~\ref{lem1} to simulation-based planning 
can seem counter-intuitive, and very much in spirit of the aphorism 
\emph{all models are wrong, but some are useful}:

\newtheorem{remark}{Remark}

\begin{remark}
  Consider the following composition $BR_i( BR_j (\pi) )$ with  $i \neq j$ for some 
  joint policy $\pi$; the interpretation is that agent $j$ first adapts to the policies $\pi_{-j}$, 
  including the $i$-th agent's policy $\pi_i$; then agent $i$ adapts to the policies $(BR_j(\pi))_{-i}$.
      The subsequent application of $BR_i$ on $BR_j (\pi)$ is likely to update agent $i$'s policy, 
      which means that the assumption agent~$j$ made to compute its best-response (namely that $i$ uses $\pi_i$) is no longer true.
    Nevertheless, as shown by Lemma~\ref{lem1}, the value of the joint policy still increases.
\end{remark}

\newtheorem{cor}{Corollary}

\begin{definition}
  Let $\sigma$ be a permutation on the set $\{1,\dots,n\}$. We define the joint response operator by
  $JR_{\sigma}:= BR_{\sigma(n)} \circ \dots \circ BR_{\sigma(1)}$.
\end{definition}

\begin{cor}\label{cor1}
  For all permutations $\sigma$ and all initial joint policies the 
  iterative application of operator $JR_{\sigma}$ converges to a Nash equilibrium.
  Since the policy space is finite, the convergence is achieved in finite time.
\end{cor}

\begin{proof}
  To make the argument easier to follow,
  we will assume that $\sigma=\operatorname{id}$, and denote $JR_{\operatorname{id}}$
  as $JR$. For the purpose of this proof we denote the $N$-th composition of $JR$
  by $JR^N$, for any $N \in \mathbb{N}$.
  Since $V^{M,JR(\cdot)}(s)$ is non-decreasing as a function of joint policies, and the policy set is finite,
  for any joint policy $\pi$ there exists an $N \in \mathbb{N}$ such that 
  $V^{M,JR^{(N+1)}(\pi)}(s)=V^{M,JR^N(\pi)}(s)\ \forall s$.
  We will show that $\pi^N:=JR^N(\pi)$ is a Nash equilibrium.
  Since $V$ increases along trajectories generated by $BR_i$, we have 
  \begin{equation}
    V^{M,JR^N(\pi)}(s)=V^{M, BR_1(JR^N(\pi))}(s)\ \forall s.
  \end{equation}
  Agent 1 has no incentive to deviate (c.f. the second part of Lemma~\ref{lem1}), so
  \begin{equation}
    JR^N(\pi)=BR_1(JR^N(\pi)).
  \end{equation}
  By an inductive argument
  \begin{equation}
    BR_i(JR^N(\pi))=JR^N(\pi).
  \end{equation}
  for all $i \in \{1,\dots ,n\}$,
  which concludes the proof.
\end{proof}

\subsection{Behavioral Cloning}
In an online planning setting, accessing the policies of other agents
can be computationally expensive,
especially if the policies of individual agents are formed by executing an algorithm which evaluates
the Q values ``on-the-go'' -- such as in MCTS. 
To address this issue, we
propose to use machine learning models.

More precisely, we divide our policy improvement process into generations;
at each generation we update models in the simulator of one of our agents 
with samples of other agents' policies from the previous generations.
Through machine learning we are able to extrapolate and
to give predictions of actions for states that were unseen during the previous simulation runs
(i.e. were never explored by the policies of previous generations).
By our method, each agent $i$ uses MCTS to
conduct its own individual planning in the environment $\Pi_i(M,\pi^h_{-i})$,
where $ \pi^h_{-i} $ is a model of policies of other agents.
Therefore, the planning is fully decentralized, and no communication is needed to execute the policies.

\subsection{The ABC Pipeline}
Our policy improvement pipeline based on MCTS 
is presented in pseudocode in Algorithm~\ref{alg2}.
The behavioral cloning algorithm is presented in pseudocode in Algorithm~\ref{alg3}.

\begin{algorithm}[tb]
    \caption{The ABC policy improvement pipeline.}\label{alg2}
    \textbf{Inputs:}
    MMDP $M$ \\
    initial heuristic policies $\pi^{h,0}=(\pi^{h,0}_1,\dots,\pi^{h,0}_n)$ \\
    \textbf{Parameters:} number of generations $nGen$\\ MCTS parameters $l$, $C$ \\
    \mbox{\textbf{Outputs:} improved policies $\pi^{nGen}$}
    \textbf{Start:}
    \begin{algorithmic}[1]
   \STATE $\forall i: \pi^0_i := MCTS \left(\Pi_1(M,\pi^{h,0}_{-i}),C,l,\pi_i^{h,0} \right)$ 
    \FOR{g in 1:nGen}
    \STATE    perform simulation with $\pi^{g-1}$, collect data $d$     
    \STATE    $\forall i:$ train new teammate models $\pi_i^{h,g-1} \approx \pi_i^{g-1}$\\
    (by Algorithm~\ref{alg3} with data $d$)
    \STATE $\forall i:~~~  \pi^g_i := \pi^{g-1}_{i}$   \hfill//copy previous policies      
    \STATE    $j := (g \operatorname{mod} n)+1 $       \hfill//agent to update  
     \STATE   $\pi^g_{j} := \operatorname{MCTS}
        \left(
            \Pi_{j}
            \left(M,\pi^{h,g-1}_{-j}\right),C,l, \pi_j^{h,g-1}
        \right)
        $ 
        \ENDFOR
      \end{algorithmic}
\end{algorithm}

\begin{algorithm}[tb]
   \caption{The algorithm for training policy approximators.}\label{alg3}
    \textbf{Inputs:}
  data $d = (s^g_i,a^g_i)_{i,g}$ of state-action pairs indexed by agents $i$, in generation $g$\\ 
  neural network policy models $\pi^{h,g}_i(\theta_0,\cdot):\ S \to [0,1]^{A_i}$, with softmax outputs over the action space\\
  \mbox{\textbf{Outputs:} trained policy models $\pi^{h,g}_i$}
  \textbf{Start:}\begin{algorithmic}[1]
  \FOR{i in 1:nAgents}
   \STATE Convert states $s^g_i$ to arrays
   \STATE One-hot-encode actions $a^g_i$
   \STATE Initialize neural network policy approximators with weights $\theta:=\theta_0$
   \FOR{e in 1:nTrainingEpochs}
    \STATE draw batch $B_g$ 
    \STATE minimize $- \sum_{ (s^g_i,a^g_i) \in B_g } a^g_i \log\left(\pi^{h,g}_i(\theta,s)\right)$ over $\theta$ (cross-entropy)
   \ENDFOR
   \ENDFOR
 \end{algorithmic}
\end{algorithm}

Since MCTS with UCT converges to actual Q values, we can conclude that for sufficient number of UCT iterations
our algorithm indeed executes the best responses to the assumed policies:

\begin{cor}
Let $\pi^g$ be as in Algorithm~\ref{alg2}.
For $l$ large enough 
$\pi^g_i=\pi^{*, \Pi_i\left(M,\pi^{h,g-1}_{-i}\right)}$, and as consequence $\pi^g=BR_i(\pi^{h,g-1})$.
\end{cor}

If our machine learning model has enough data and degrees of freedom to perfectly replicate policies,
from Lemma~\ref{lem1} and Corollary~\ref{cor1} we conclude that the procedure 
improves joint policies and eventually converges to a Nash equilibrium:

\begin{theorem}
  For $l$ large enough
  and under assumption $\pi^{h,g}=\pi^g,\ \forall g$,
  the joint policy value $V^{\pi^g,M}(s_0)$ is non-decreasing as a function of $g$, 
  and strongly increasing until it reaches the Nash equilbrium. For 
  $N$ large enough, 
  $\pi^N$ is a Nash equilibrium.
\end{theorem}

We emphasize that it is essential that only one agent updates its assumed policies at each generation.
If two or more agents would simultaneously update their policies,
they could enter an infinite loop,
always making the wrong assumptions about each other in each generation, 
and never achieving the Nash equilibrium.

We remark that in our algorithm we also leverage a learned model of agents' own
behavior by employing it in MCTS rollout stage.

\section{Experiments}

Our work is a natural extension to~\cite{claes2015effective,claes2017decentralised},
we perform experiments on a slightly modified version of the Factory Floor domain introduced therein.
The baseline for our experiments is given by the current state of the art planning method for this domain:
individual MCTS agents 
with heuristic models of other agents~\cite{claes2017decentralised},
and it also serves as initialization (generation 0) of the ABC policy iterator.
Therefore, the goal of the experiments 
is to empirically confirm the policy improvement via ABC.
Any improvement over the 0th generation shows that we have managed to beat the baseline.

\subsection{The Factory Floor Domain}

The domain consists of a gridworld-like planar map, where each position
can be occupied by (cleaning) robots and tasks (e.g. litter).
Multiple robots and/or tasks can be in the same position.
Each robot is controlled by an agent, and at each time step an agent can perform either a
movement action $UP, DOWN, LEFT, RIGHT$, which shifts the position of the robot accordingly,
or a cleaning action $ACT$, which removes one task at the current position.
Attempted actions may succeed or not, according to predefined probabilities.
The reward collected at each time step is the number of tasks cleaned by the robots.
At the beginning of the simulation there can already be some tasks on the map,
and, as the simulation progresses, more tasks 
can appear, according to predefined probabilities.

\subsection{Initial Heuristic Models}

Below, we describe the heuristic policies $\pi^{h,0}_i$,
which are supplied
as the model for MCTS agents in generation $0$ (the baseline).
At each decision step, $i$-th heuristic agent acts according to the following recipe:
\begin{enumerate}
  \item it computes the \emph{social order} of the corresponding 
    robot among all robots sharing the same position;
    the social ordering function is predefined by the lexicographic order of unique robot identifiers.
  \item It evaluates each possible destination $\tau$ by the following formula:
    \begin{equation}
      NV(\tau, \text{robot}_i)=
    \begin{cases}
      -\infty \text{ if no tasks at }\tau,\\
      \frac{\# \text{tasks}}{\operatorname{dist}(\tau,\text{robot}_i)};
    \end{cases}    
  \end{equation}
  \item It assigns the $k$-th best destination as the target destination, 
    where $k$ is the computed social order 
    of the corresponding robot
    (e.g. if it is the only robot at a given position, then $k=1$).
    Therefore, the social order is used to prevent several agents choosing the same destination.
  \item It chooses action $ACT$ if it is already at the target destination;
    and otherwise it selects a movement action along the shortest path to the destination.
\end{enumerate}

\subsection{MCTS Settings}

We scale the exploration constant $C$ by the remaining time steps in the simulation, i.e. $c =c(t):= C*(H-t)$,
to account for the decreasing range of possible future rewards, as recommended in~\cite{kocsis2006bandit}.
As in the baseline, 
we also use sparse UCT~\cite{bjarnason2009lower} to combat the problem of a large state space; 
that means that we stop sampling child state nodes 
of a given action node from the simulator after we have sampled a given amount of times;
instead we sample the next state node from the existing child state nodes,
based on frequencies with which they occured.
In all our experiments, we set this sampling limit to $20$.
As in the baseline, the agents are awarded an additional do-it-yourself
bonus of $0.7$ in simulation, if they perform the task themselves;
this incentivizes them to act, rather than rely on their teammates.
Each agent performs 20000 iterations of UCT to choose the best action for their robot.

\subsection{The Behavioral Cloning Model}

Since the domain has spatial features,
we opted to use a convolutional neural network
as the machine learning method of choice for policy cloning.

As input we provide a 3-dimensional tensor with the width and the height equal
to the width and the height of the Factory Floor domain grid,
and with $n+2$ channels (i.e. the amount of robots plus two).
We include the current time step information in the state.
The $0$-th channel layer is a matrix filled with integers representing
the amount of tasks at a given position.
The tasks have finite execution time, 
and the current time step 
affects the optimal decision choice; 
therefore we encode the current time step by filling it in the entries
of the 1st channel.
Finally, for $i=1,\dots, n$, 
the $2+i$-th channel is encoding the position of robot $i$, by setting $1$ where the robot is positioned
and $0$ on all other fields.

Such state representation is fed into the neural network with two convolutional layers
of 2x2 convolutions followed by three fully connected layers with 64, 16 and 5 neurons respectively. 
We use the rectified linear unit activation functions between the layers,
except for the activation of the last layer, which is given by the softmax activation function.
The network has been trained using the categorical cross entropy function
as the loss function, and Adam as the optimization method~\cite{kingma2014adam}.
The action assigned to the state during MCTS simulations
is corresponding to the argmax coordinate of the softmax probabilities.
The time required to train the neural network is insignificant,
compared to the time needed to collect data from MCTS simulations.

\subsection{Domain Initialization}

We tested our method in four experiments.
In all experimental subdomains, the movement actions are assumed to succeed with probability $0.9$,
and the $ACT$ action is assumed to succeed always.
In all configurations the horizon $H$ is set to ten steps,
and the factor $\gamma$ is set to $1$, so there is no discounting of future rewards.
We present the initial configuration of the experiment
and the corresponding reward plots in Figures~\ref{exp1},~\ref{exp2}, and~\ref{exp3}.
Letters $R$ indicate robot positions,
and the numbers indicate the amount of tasks at a given position -- for a fixed task placement;
or the probability that a task appears at a given position -- for dynamic task placement.
We provide plots of the results, that contain the mean average reward for each generation,
together with 95\% confidence interval bars.

We chose domain configurations which, due to the location of tasks,
require high level of coordination between agents.
In particular, we created the domains where we expect that 
the policies of the baseline are suboptimal.
For more generic domains, 
the decentralized MCTS with heuristic models is already close to optimal,
and we do not expect much improvement.
In subdomains with fixed positions of tasks we train the agents
for five generations.
In subdomains, where the tasks are assigned dynamically,
we train the agents for three generations, as for higher amount of iterations 
we sometimes observed worsening performance, 
which we attribute to imperfect learning process due to high stochasticity of the domain.

\paragraph{Two robots.} Our first subdomain is a trivial task: a 6x4 map, which has eight tasks to be collected.
Even in such a simple scenario, 
the baseline does not perform well, because both robots
make the assumption that their colleague will serve the task piles of $2$'s and head for the $1$s,
achieving a mean reward of $\approx 5.5$ (0th generation).
The exploration parameter $C$ is set to $0.5$, and the number of simulations at each generation $nSim$
to 320.
Already in the first generation, agent $2$ learns the policy of agent $1$ and adapts accordingly,
which results in an increase of the mean collected reward to $\approx 7.9$.
The average collected reward stabilizes through the next generations,
which suggests that our method reached a Nash equilibrium 
(and in fact a global optimum, given that the maximal reward that could have been obtained 
in each episode is $8$).

\begin{figure}[t]
  \centering
  \includegraphics[width=37mm]{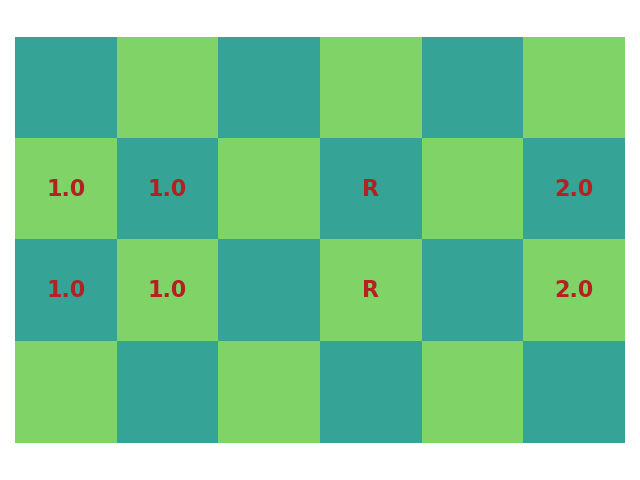}\qquad
  \includegraphics[width=41mm]{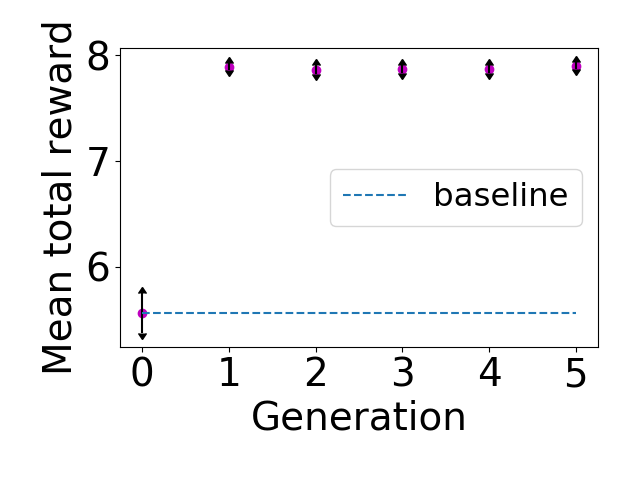}
\caption{Left: The two robots experiment map. Right: Mean rewards of the experiment with two robots. Each generation represents 320 simulations.}
\label{exp1}
\end{figure}

\paragraph{Four robots, fixed tasks.} Our second subdomain is a 7x7 map which has 22 tasks to be collected by four robots.
The exploration parameter $C$ is increased to $1.0$ -- to account for higher possible rewards, 
and the number of simulations at each generation $nSim$ is decreased to $180$ 
-- to account for longer simulation times.
All robots start from the middle of the map.
The baseline method again underpeforms,
as the robots are incentivized to go for the task piles of 3's and 4's,  
instead of spreading in all four directions.
After the application of the ABC algorithm the robots learn the directions of their teammates, spread, 
and a near optimal learning performance is achieved, see Figure~\ref{exp2}.

\begin{figure}[t]
  \centering
  \includegraphics[width=41mm]{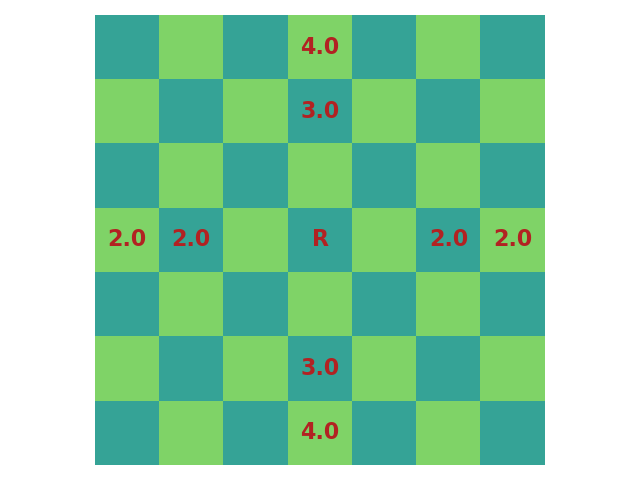}\quad \includegraphics[width=41mm]{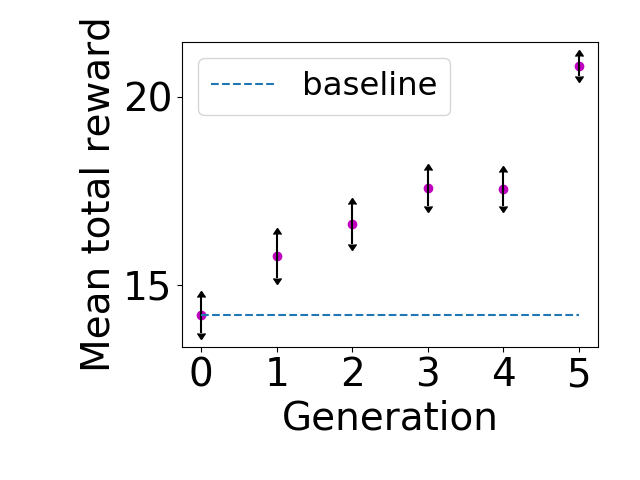}
\caption{Left: The four robots experiment map with fixed task allocation. All robots start in the middle.
Right: Mean rewards of the four robots experiment with fixed task positions.
  Each generation represents 180 simulations.}
\label{exp2}
\end{figure}

\begin{figure}[t]
  \centering
  \includegraphics[width=50mm]{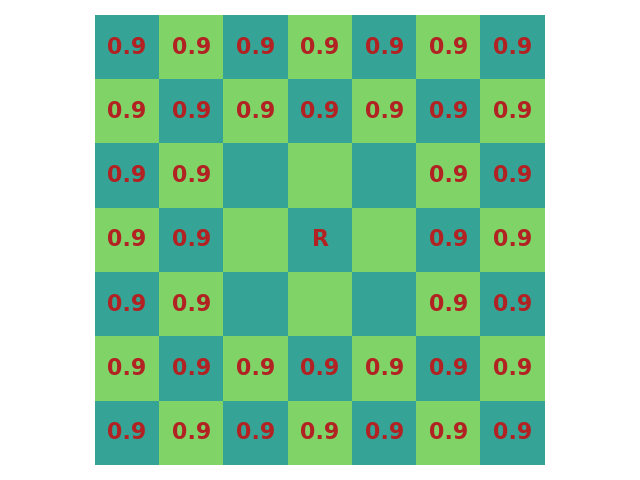}

  \vspace{0.7cm}

  \includegraphics[width=4.1cm]{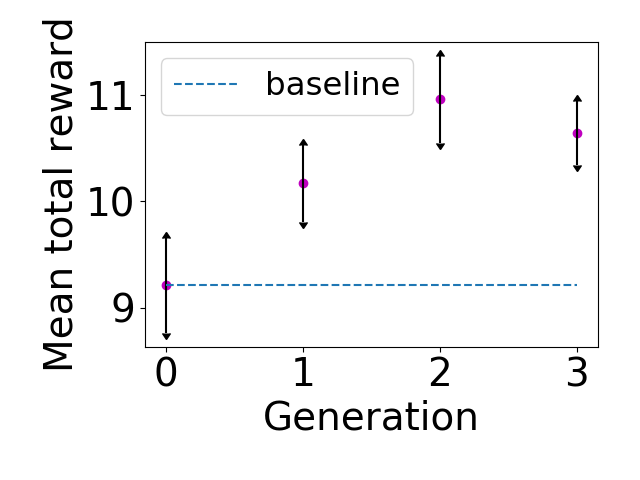}\quad  \includegraphics[width=4.1cm]{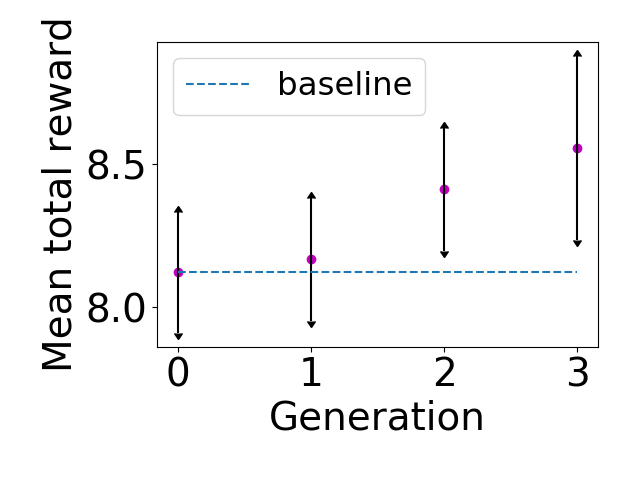}
\caption{Top: The four robots experiment map with dynamic task allocation. All the robots start in the middle.
Bottom: Mean rewards of the dynamic task assignment experiment with four robots and two (left) / three (right) tasks appearing at each
  time step -- data from 90/180 simulations respectively.}
\label{exp3}
\end{figure}

\paragraph{Four robots, dynamic tasks.} For the final two experiments we chose the same 7x7 map as previously,
but this time tasks appear dynamically: two or three new tasks are added randomly with probability 0.9
at each time step during the program execution in one of the marked places.
All the other experiment parameters remain unchanged.
The confidence intervals are wider,
due to additional randomness.
Nevertheless, for the first three generations we observed an improvement
over the 0th generation,
which we attribute to the fact that the agents have learnt that they should spread to cover the task
allocation region,
similarly as in the experiment with fixed task location.

\section{Conclusions}

We have proposed a machine-learning-fueled method of improving teams of MCTS agents.
Our method is grounded in the theory of alternating maximization and, given sufficiently rich
training data and suitable planning time, it is guaranteed to improve the initial joint policies
and reach a local Nash equilibrium.
We have demonstrated in experiments, that the method allows to improve team policies for 
spatial task allocation domains, where coordination is crucial to achieve optimal results.

An interesting direction of future work is to search for the global optimum
of the adaptation process, rather than a local Nash equilibrium.
To that end, one can randomize the orded in which agents are adapting,
find multiple Nash equilibria, and select the one with highest performance.
Another research avenue is to extend the ABC method to environments with partial information
(Dec-POMDPs), where the agents
need to reason over the information set available to their teammates.

\section*{Acknowledgments}
This project received funding from 
 EPSRC First Grant EP/R001227/1, and 
 the European Research Council (ERC) 
under the European Union's
Horizon 2020 research 
and innovation programme (grant agreement No.~758824 \textemdash INFLUENCE).
 \begin{figure}[h!]
   \centering
\includegraphics[width=0.4\columnwidth]{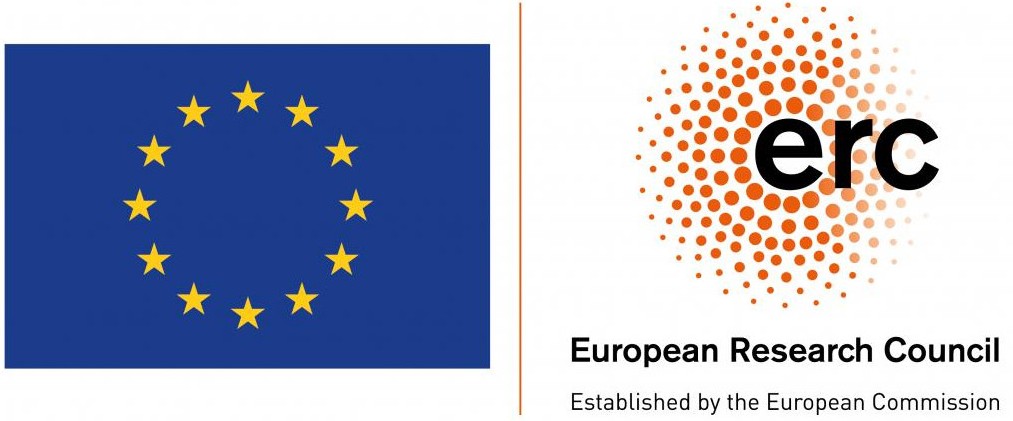}
\end{figure}

\bibliographystyle{named}
\bibliography{ijcai20}

\end{document}